\documentclass[sigconf]{acmart}

\makeatletter
\def\@ACM@checkaffil{
    \if@ACM@instpresent\else
    \ClassWarningNoLine{\@classname}{No institution present for an affiliation}%
    \fi
    \if@ACM@citypresent\else
    \ClassWarningNoLine{\@classname}{No city present for an affiliation}%
    \fi
    \if@ACM@countrypresent\else
        \ClassWarningNoLine{\@classname}{No country present for an affiliation}%
    \fi
}
\makeatother

\pdfoutput=1 

\usepackage{graphicx}

\usepackage{subfigure}
\usepackage{amsmath,amsfonts,amsthm,bm}
\usepackage{amsfonts}
\usepackage{amsxtra}
\usepackage{amsthm}

\usepackage{pgfplots}
\usepgfplotslibrary{statistics}

\usepackage{multicol}
\usepackage{tabu}

\usepackage{adjustbox}

\usepackage{multirow}
\usepackage{flushend}

\usepackage{float}

\usepackage{caption}

\usepackage{calc}

\usepackage{cancel}

\usepackage{mathrsfs}
\usepackage{algorithm}
\usepackage{algorithmic}

\newtheorem{definition}{Definition}

\newtheorem{proposition}{Proposition}

\newtheorem{example}{Example}
\newtheorem{problem}{Problem}

\usepackage{cuted}
\usepackage{blindtext, graphicx}
\usepackage{amsmath,float}
\usepackage{dirtytalk}
\usepackage{url}
\usepackage{amsmath,amsfonts,amsthm,bm}
\usepackage[utf8]{inputenc}
\usepackage[english]{babel}
\usepackage{suffix}
\usepackage{mathtools}
\usepackage{multicol}
\usepackage{tabu}

\usepackage{caption}

\usepackage{graphics}
\usepackage{adjustbox}
\usepackage{lipsum}
\usepackage{stfloats}
\usepackage{multirow}

\usepackage{flushend}
\usepackage{tikz,colortbl}
\usepackage{zref-savepos}
\usepackage{pifont}

\usetikzlibrary{calc}

\allowdisplaybreaks

\usepackage{soul}

\usepackage{color}

\newcommand{\NN}{\mathcal{N}\!\mathcal{N}}

\newcounter{NoTableEntry}
\renewcommand*{\theNoTableEntry}{NTE-\the\value{NoTableEntry}}

\renewcommand\footnotetextcopyrightpermission[1]{}

\title{
BERN-NN: 
Tight Bound Propagation For Neural Networks Using Bernstein Polynomial Interval Arithmetic}

\author{Wael Fatnassi}\authornote{Both authors contributed equally to the paper}
\affiliation{
	\institution{University of California, Irvine}
	\department{Dept. of Electrical Engineering and Computer Science}
}
\email{wfatnass@uci.edu}

\author{Haitham Khedr}\authornotemark[1]
\affiliation{
	\institution{University of California, Irvine}
	\department{Dept. of Electrical Engineering and Computer Science}
}
\email{hkhedr@uci.edu}

\author{Valen Yamamoto}
\affiliation{
	\institution{University of California, Irvine}
	\department{Dept. of Electrical Engineering and Computer Science}
}
\email{vyamamot@uci.edu}

\author{Yasser Shoukry}
\affiliation{
	\institution{University of California, Irvine}
	\department{Dept. of Electrical Engineering and Computer Science}
}
\email{yshoukry@uci.edu}


\begin{document}\sloppy


\begin{abstract}
In this paper, we present BERN-NN as an efficient tool to perform bound propagation of Neural Networks (NNs). Bound propagation is a critical step in wide range of NN model checkers and reachability analysis tools. Given a bounded input set, bound propagation algorithms aim to compute tight bounds on the output of the NN. So far, linear and convex optimizations have been used to perform bound propagation. Since neural networks are highly non-convex, state-of-the-art bound propagation techniques suffer from introducing large errors. To circumvent such drawback, BERN-NN approximates the bounds of each neuron using a class of polynomials called Bernstein polynomials. Bernstein polynomials enjoy several interesting properties that allow BERN-NN to obtain tighter bounds compared to those relying on linear and convex approximations. BERN-NN is efficiently parallelized on graphic processing units (GPUs). Extensive numerical results show that bounds obtained by BERN-NN are orders of magnitude tighter than those obtained by state-of-the-art verifiers such as linear programming and linear interval arithmetic. Moreoveer, BERN-NN is both faster and produces tighter outputs compared to convex programming approaches like alpha-CROWN. 
\end{abstract}
\maketitle
\pagestyle{plain}

\section{Introduction}
Neural Networks (NNs) have become an increasingly central component of modern, safety-critical, cyber-physical systems like autonomous driving, autonomous decision-making in smart cities, and even autonomous landing in avionic applications.
Thus, there is an increasing need to verify the safety and correctness ~\cite{sun2019formal,sun2021provably,fremont2020formal} of NNs when they are used to control physical systems.

The problem of NN Verification has been well studied in literature~\cite{liu2021algorithms}. Most NN verifiers rely mainly on either using linear relaxation and optimization~\cite{wang2018efficient, dvijotham2018dual, wong2017provable,henriksen2021deepsplit,khedr2021peregrinn,wang2021beta} to falsify a given property or prove its satisfaction, or reachability analysis to compute an over-approximation of the output set. The latter is specifically important for control applications where the property of interest is defined over a time horizon. Both techniques rely on overapproximation, hence, having tight output bounds is at the core of NN verification as it allows reasoning about NN properties in an efficient manner. For example, model checking the robustness of NNs against adversarial perturbations can be done by simply comparing the tight bounds of the outputs of the network. Moreover, networks used in control applications often involve multi-step reachability, and hence computing tight bounds is crucial to harness the accumulation of the error and hence be able to efficiently reason about the safety of the system.

Due to the non-convexity and non-linearity of NNs, the problem of finding the exact bounds of NN outputs is NP-hard\cite{KatzReluplexEfficientSMT2017a}. Different tools have been proposed to find tight overapproximations of NN outputs. MILP-based methods~\cite{dutta2019sherlock,lomuscio2017approach, tjeng2017evaluating, bastani2016measuring, bunel2020branch, fischetti2018deep, anderson2020strong, cheng2017maximum} encode the non-linear activations as linear and integer constraints. Reachability methods~\cite{xiang2017reachable, xiang2018output, gehr2018ai2, wang2018formal, tran2020nnv, ivanov2019verisig, fazlyab2019efficient} use layer-by-layer reachability analysis (exact or overapproximation) of the network. Most of these methods either rely on convex \emph{linear relaxation} of the non-linear activation functions to overapproximate the output of the NN, or try to find the exact bounds which are often intractable.

In this work, we explore using polynomials to approximate non-linear activations (e.g. ReLU). More specifically, we approximate non-linear activations using Bernstein polynomials which are constructed as a linear combination of the Bernstein basis polynomials~\cite{farouki2012bernstein}. The use of Bernstein polynomials is motivated by two reasons. First, based on the Stone-Weierstrass approximation theorem~\cite{de1959stone}, Bernstein polynomials can uniformly approximate continuous activation functions. Second and most importantly, bounding a Bernstein polynomial is computationally cheap based on the interesting properties of Bernstein polynomials discussed in section \ref{sec:motivation}. The goal of using higher-order polynomials versus linear relaxation is to get tight bounds on NNs which is crucial for verifying a large class of formal properties. This idea of using polynomials has inspired other researchers~\cite{dutta2019reachability,fan2020reachnn,huang2022polar}, however, the proposed tools suffer from scalability issues.

Our main contributions can be summarized as follows:
\begin{itemize}
    \item We propose a tool that uses Bernstein polynomials to approximate ReLU activations and hence compute tighter NN bounds than state-of-the-art.
    \item The tool is designed with scalability in mind; hence, the entire operations can be accelerated using GPUs.
    \item We show that by using the proposed approximation, we are able to compute tighter output sets than alpha-Crown (winner of VNN22' competition\cite{bak2021second} for Formal Verification of NNs) and other state-of-the-art bounding methods. For instance, BERN-NN approximations are twice reduced compared to alpha-Crown for actual NN's controllers. Moreover, Numerical results showed that Bern-NN can process neural networks with more than 1000 neurons in less than 2 minutes
\end{itemize}

\section{Problem Formulation}
\subsection{Notation:}
\noindent \textbf{General notation:}
We use the symbols $\mathbb{N}$ and $\mathbb{R}$ to denote the set of natural and real numbers, respectively.
We denote by $x=\big(x_1,x_2,\cdots,x_n\big) \in \mathbb{R}^n$ the vector of $n$ real-valued variables, where $x_i \in \mathbb{R}$. We denote by $I_n (\underline{d}, \overline{d}) =\big[\underline{d}_1,\overline{d}_1\big] \times \cdots \times$ $\big[\underline{d}_n,\overline{d}_n\big] \subset \mathbb{R}^{n}$ the $n$-dimensional hyperrectangle where $\underline{d} = \left(\underline{d}_1, \cdots, \underline{d}_n\right)$ and $\overline{d} = \left(\overline{d}_1, \cdots, \overline{d}_n\right)$ are the lower and upper bounds of the hyperrectangle, respectively. 
We denote by $x^T$ and $A^T$ the transpose operation of the vector $x$ and the matrix $A$. We denote by $0_n$ a vector that contains $n$ zero values and by $0_{n \times m}$ the matrix of shape $n \times m$ that contains zeros. Finally, $A * B$ stands for the element-wise product between the multi-dimensional tensors $A$ and $B$, and $A \otimes B$ stands for the Kronecker product between the matrices $A$ and $B$.

\noindent \textbf{Notation pertaining to multivariate polynomials:}
For a real-valued vector $x =\big(x_1,x_2,\cdots,x_n\big)\in \mathbb{R}^n$ and an index-vector $K = \left(k_1, \cdots, k_n\right) \in \mathbb{N}^n$, we denote by $x^K \in \mathbb{R}$ the scalar $x^K = x_1^{k_1} \times \ldots \times x_n^{k_n}$. 
Given two multi-indices $K = \left(k_1, \cdots, k_n\right) \in \mathbb{N}^n$ and $L = \left(l_1, \cdots, l_n\right) \in \mathbb{N}^n$, we use the following notation throughout this paper: 
\begin{align*}
K + L &= \left(k_1+l_1, \cdots, k_n + l_n\right), \\
{L \choose K} & ={l_1 \choose k_1} \times \cdots \times {l_n \choose k_n}, \\
\sum\limits_{K \leq L} & = \sum\limits_{k_1 \leq l_1}^{}\cdots \sum\limits_{k_n \leq l_n}
\end{align*}
%
Finally, a real-valued multivariate polynomial $p:\mathbb{R}^n \rightarrow \mathbb{R}$ is defined as:
\begin{align*}
 p(x_1, \ldots, x_n) & \;=\; \sum_{k_1 = 0}^{l_1 } \sum_{k_2 = 0}^{l_2} \ldots \sum_{k_n = 0}^{l_n} a_{(k_1,\ldots,k_n)} x_1^{k_1} x_2^{k_2} \ldots x_n^{k_n} \nonumber \\
 &\;=\;\sum\limits_{K \leq L} a_K x^K,
\end{align*}
where $L = (l_1, l_2, \ldots, l_n)$ is the maximum degree of $x_i$ for all $i = 1, \ldots, n$. 

\noindent \textbf{Notation pertaining to neural networks:}
In this paper, we consider $H$-layer, feed-forward, ReLU-based neural networks $\mathcal{N}\!\mathcal{N}: \mathbb{R}^n \rightarrow \mathbb{R}^o$ defined as:
\begin{align*}
    \mathcal{N}\!\mathcal{N}(x) &= W^{(H)} z^{(H-1)} + b^{(H)} \\
    z^{(H-1)} &= \sigma\left(W^{(H-1)} z^{(H-2)} + b^{(H-1)} \right) \\
    \vdots \\
    z^{(1)} &= \sigma\left(W^{(1)} x + b^{(1)} \right)
\end{align*}
where $\sigma$ is the ReLU activation function (i.e., $\sigma(z) = \max(0,z)$) that operates element-wise, $W^{(i)} \in \mathbb{R} ^ {h_i \times h_{i - 1}}$ and $b^{(i)} \in \mathbb{R}^{h_i}$ with $i \in \{1, \cdots, H\}$ are the weights and the biases of the network. For simplicity of notation, we use $\hat{z}^{(i)}_{j}$ and $z^{(i)}_{j}$ to denote the pre-activation (input) and the post-activation (output) of the $j$-th neuron in the $i$-th layer.


\subsection{Main Problem:} In this paper, we seek to find 
polynomials that upper and lower approximate the NN's outputs $\NN(x)$ whenever the NN's input $x$ is confined within a pre-defined hypercube, i.e. $x \in I_n (\underline{d}, \overline{d})$.

\begin{problem}
Given a neural network $\NN:\mathbb{R}^n \rightarrow \mathbb{R}^o$ and an input domain hypercube $I_n (\underline{d}, \overline{d}) \subset \mathbb{R}^n$. Find lower and upper approximate polynomials $\left(\underline{p}_{\NN,1}(x), \overline{p}_{\NN,1}(x)\right), \ldots \left(\underline{p}_{\NN,o}(x), \overline{p}_{\NN,o}(x)\right)$, such that:
\begin{align*}
\underline{p}_{\NN,1}(x) & \leq \NN_1(x) \leq \overline{p}_{\NN,1}(x) \\
& \qquad \qquad \vdots \\
\underline{p}_{\NN,o}(x) & \leq \NN_o(x) \leq \overline{p}_{\NN,o}(x), 
\end{align*}
where with some abuse of notation, we use $\NN_i(x)$ to denote the $i$th output of the neural network $\NN$.

\end{problem}
Note that the lower/upper bound polynomials $\left(\underline{p}_{\NN,1}(x), \overline{p}_{\NN,1}(x)\right), \ldots \left(\underline{p}_{\NN,o}(x), \overline{p}_{\NN,o}(x)\right)$ depend on the input domain $I_n$. That is, for each value of $I_n$, we need to find different lower/upper bound polynomials. However, for the sake of simplicity of notation, we drop the dependency on $I_n$.

\section{Tight bounds of ReLU Functions Using Bernstein Polynomials}
\label{sec:motivation}
To solve Problem 1, we rely on a class of polynomials called Bernstein polynomials which are defined as follows: 
\begin{definition} (Bernstein Polynomials) 
Given a continuous function $g:\mathbb{R}^n \rightarrow \mathbb{R}$, an input domain (hypercube)  $I_n(\underline{d}, \overline{d}) \subset \mathbb{R}^n$, and a multi-index $L = \left(l_1, \cdots, l_n\right) \in \mathbb{N}^n$, the polynomial: 
\begin{align}
B_{g, L}\left(x\right) &= \sum\limits_{K \leq L}^{} b^{g}_{K,L} Ber_{K, L}\left(x\right), \label{bernpol} \\
Ber_{K, L}\left(x\right) &= {L \choose K} \frac{\left(x - \underline{d}\right)^{K}\left(\overline{d} - x\right)^{L-K}}{\left( \overline{d} - \underline{d}\right)^L}, \label{bernpolcoeff}\\ 
b^{g}_{K, L} &= g\bigg(\left(\overline{d}_1 - \underline{d}_1\right) \frac{k_1}{l_1} + \underline{d}_1, \cdots, \left(\overline{d}_n - \underline{d}_n\right) \frac{k_n}{l_n} + \underline{d}_n \bigg), \label{bernpolcoeff2}
\end{align}
is called the $L$th order Bernstein polynomial of $g$, where $Ber_{K, L}\left(x\right)$ and $b^g_{K, L}$ are called the Bernstein basis and Bernstein coefficients of $g$, respectively.
\end{definition}

Bernstein polynomials are known to be capable of approximating any continuous function. That is, Bernstein approximation has an advantage compared to Taylor approximation because the latter relies on the function being differentiable. In this case, Taylor model can not approximate ReLU activation functions because they are not differentiable which makes Bernstein polynomials a good option to approximate ReLU functions. Bernstein polynomials have an interesting and useful property called \textit{range enclosing property} which is defined as follows:

\begin{definition}(Range Enclosing Property \cite{range_enclose_prop})
Given a multi-dimensional polynomial $p\left(x\right)$ of order $L$ that it defined over the region $I_n\left(\underline{d}, \overline{d}\right)$ with its Bernstein polynomial $B_{p, L} = \sum\limits_{K \leq L}b^{p}_{K, L}\left(x\right)Ber_{K, L}\left(x\right)$. The following holds for all $x \in I_n\left(\underline{d}, \overline{d}\right)$:
\begin{align}\label{range_enclosing_property}
 \min\limits_{K \leq L} b^{p}_{K, L} \; \leq \; p\left(x\right) \; \leq \; \max\limits_{K \leq L} b^{p}_{K, L}.   
\end{align}
\end{definition}

The range enclosing property states that the minimum (maximum) over all the Bernstein coefficients is a lower (upper) bound for the polynomial $p$ over the region $I_n\left(\underline{d}, \overline{d}\right)$. These bounds provided by the Bernstein coefficients are generally tighter than those given by interval arithmetic and many centered forms~\cite{bern_better_ia}. Note that the range enclosing property applies only when the Bernstein polynomial is used to approximate other polynomials $p$ and other continuous functions $g$. Nevertheless, as we show in Section 4, these bounds will be helpful to provide tight bounds on the polynomials used to over/under approximate the individual neurons and hence obtain tight polynomial bounds on the NN's outputs. 

\begin{figure}[!t]
\centering
	\includegraphics[width=0.235\textwidth,trim={16mm 0 16mm 12mm},clip]{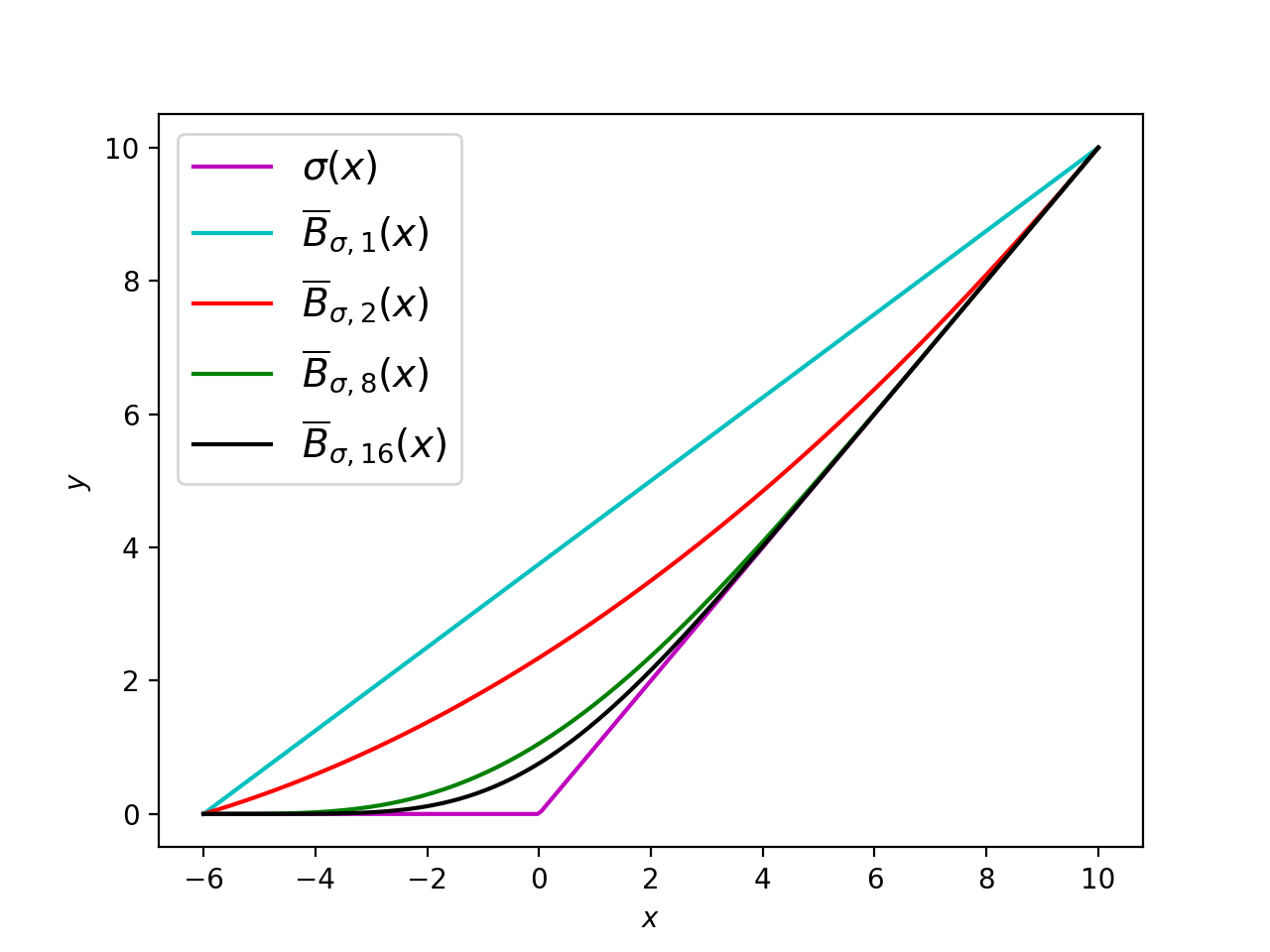} 
	\includegraphics[width=0.235\textwidth,trim={16mm 0 16mm 12mm},clip]{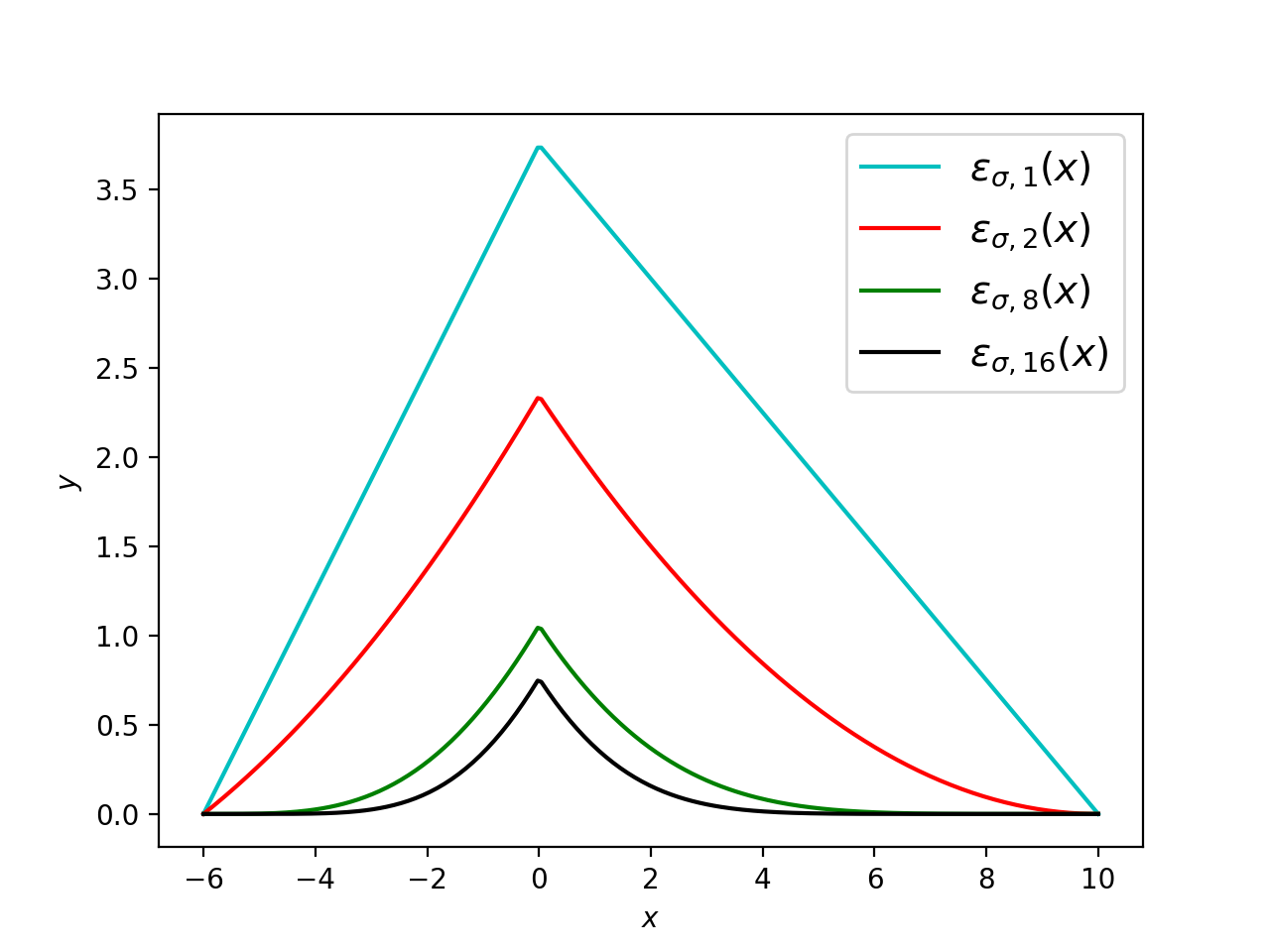} 
	\caption{\textbf{(Top)} Bernstein polynomial approximations of ReLU activation for different approximation's order $L \in \{1, 2, 8, 16\}$, in the interval $I_1\left(-6, 10\right) = \big [-6, 10\big ]$. \textbf{(Bottom)} Bernstein polynomial approximations of ReLU and their associated approximation errors for different approximation's order $L \in \{1, 2, 8, 16\}$ in the interval $I_1\left(-6, 10\right) = \big [-6, 10\big ]$.}
	\label{fig:ber_relu_order}
\end{figure}

\subsection{Over-Approximating ReLU functions using Bernstein Polynomials}

We now study how to use Bernstein polynomials to over-approximate the ReLU function $\sigma: \mathbb{R} \rightarrow \mathbb{R}$ defined as $\sigma(x) = \max(0,x)$.
While Bernstein polynomials can approximate any continuous function $g$, there is no guarantee that this Bernstein approximation is either over-approximation or under-approximation. The next result establishes an order between the ReLU function $\sigma$ and its Bernstein approximation.

\begin{proposition}\label{prop:over}
Given an interval $I_1\left(\underline{d}, \overline{d}\right) = \big[\underline{d}, \overline{d} \big]$, where $0 \in \big[\underline{d}, \overline{d} \big]$ and any approximation order $L \ge 1$. The following holds for all $x \in I_1$:
$$ \sigma(x) \le B_{\sigma,L}(x)= \overline{B}_{\sigma, L}(x).$$
\end{proposition}
\begin{proof}
This follows directly by substituting the function $\sigma$ in the definition of Bernstein polynomials~\eqref{bernpol}-\eqref{bernpolcoeff2}.
\end{proof}

In other words, Proposition~\ref{prop:over} states that the Bernstein polynomial of $\sigma$ is a guaranteed over-approximation of $\sigma$. This even holds \emph{for any approximation order $L$}. Moreover, since the approximation error between a function $g$ and its Bernstein approximation $B_{g,L}$ is known to decrease as $L$ increases~\cite{garloff}. Then another consequence of Proposition~\ref{prop:over} is that Bernstein polynomials produce a tighter over-approximation for ReLU functions as $L$ increases.


Figure~\ref{fig:ber_relu_order} emphasizes these conclusions pictorially where we show the Bernstein polynomials of $\sigma$ with orders $L = 1,2,8,16$. As shown in Figure~\ref{fig:ber_relu_order} (Left), the Bernstein polynomials $B_{\sigma, L}(x)$ for $L = \in \{1, 2, 8, 16\}$ over-approximate the ReLU activation function over the entire input range. Furthermore, the over-approximation gets tighter to the actual ReLU by increasing the approximation order $L$. We note that using $L = 1$, the resulting Bernstein polynomial produces the well-studied linear convexification of the ReLU function which is used in state-of-the-art algorithms for bounding neural networks including Symbolic Interval Arithmetic (SIA) \cite{wang2018efficient} and alpha-CROWN \cite{xu2020fast}. In other words, Bernstein polynomials can be seen as a generalization of these techniques. 


\subsection{Under-approximating ReLU functions using Bernstein polynomials}
In addition to the over-approximation of the ReLU function $\sigma$, it is essential to establish a Bernstein under-approximation of $\sigma$ which is captured by the following result.
%
%
\begin{proposition}\label{prop:under}
Given an interval $I_1\left(\underline{d}, \overline{d}\right) = \big[\underline{d}, \overline{d} \big]$, where $0 \in \big[\underline{d}, \overline{d} \big]$, then the following holds for all $x \in I_1$:
$$ \underline{B}_{\sigma, L}(x) = \overline{B}_{\sigma, L}(x) - \overline{B}_{\sigma, L}(0) \le \sigma(x).$$
\end{proposition}

\begin{proof}
To prove the result, we define the approximation error $\epsilon_{\sigma,L}$ as:
$$\epsilon_{\sigma,L}(x) = \overline{B}_{\sigma,L}(x) - \sigma(x).$$
We bound the maximum estimation error satisfies as follows:
\begin{align}
\max\limits_{x \in [\underline{d}, \overline{d}]} \epsilon_{\sigma, L}(x) &=  
\max\limits_{x \in [\underline{d}, \overline{d}]} \big( \overline{B}_{\sigma, L}(x) - \sigma(x) \big) \\
&\stackrel{(a)}{=} \max\limits_{x \in [\underline{d}, 0]} \overline{B}_{\sigma, L}(x) \\
&\stackrel{(b)}{=} \overline{B}_{\sigma, L}(0) 
\label{eq:max_error}
\end{align}
where $(a)$ follows from the fact that $\sigma(x) = 0$ for $x \in [\underline{d}, 0]$ and $\sigma(x) \ge 0$ for $x \in [0, \overline{d}]$ and hence the maximum of the equation is attained whenever $\sigma(x) = 0$. Equation $(b)$ holds from the monotnicity of $\overline{B}_{\sigma, L}(x)$ when $x \in [\underline{d}, 0]$---the monotnicity follows directly from the definition of $\overline{B}_{\sigma, L}(x)$---and hence the maximum is attained when $x = 0$. It follows from the definition of $\epsilon_{\sigma,L}(x)$ that:
\begin{align*}
    \sigma(x) &= \overline{B}_{\sigma,L}(x) - \epsilon_{\sigma,L}(x)
    \ge \overline{B}_{\sigma,L}(x) - \max\limits_{x \in [\underline{d}, \overline{d}]} \epsilon_{\sigma, L}(x) \\
    &= \overline{B}_{\sigma,L}(x) - \overline{B}_{\sigma,L}(0) = \underline{B}_{\sigma,L}
\end{align*}
which concludes the proof.
\end{proof}

Proposition \ref{prop:under} shows that the maximum error between the Bernstein over-approximation polynomial $\overline{B}_{\sigma, L}$ and the ReLU activation function $\sigma$ is equal to the value of the Bernstein polynomial at $0$, i.e., $\overline{B}_{\sigma, L}(0)$. This result has a direct consequence on the efficiency of our tool. It is enough to propagate over-approximation of the ReLU function and one can get an under-approximation directly by shifting the over-approximation polynomial.

Figure \ref{fig:ber_relu_order} (Right) emphasizes this fact pictorially. As it is shown in the figure, the maximum error $\epsilon_{\sigma, L}(x) = \overline{B}_{\sigma,L} - \sigma(x)$ is reached at $x = 0$ and is equal to $\overline{B}_{\sigma, L}\left(0\right)$.

\begin{figure*}[ht]
    \centering
	\includegraphics[width=0.32\textwidth, height = 0.17\textheight]{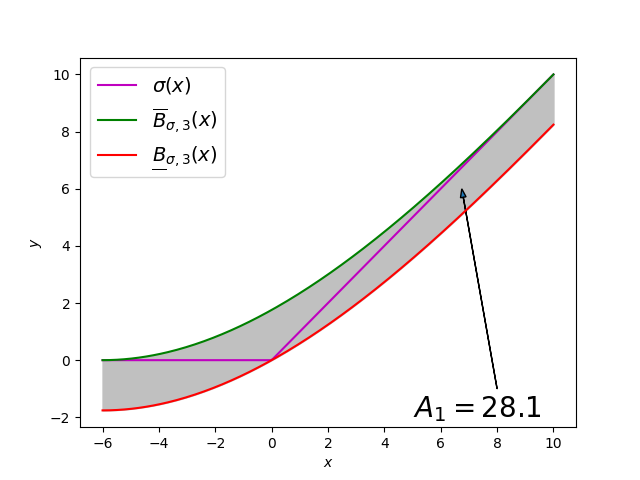} 
	\includegraphics[width=0.32\textwidth, height = 0.17\textheight]{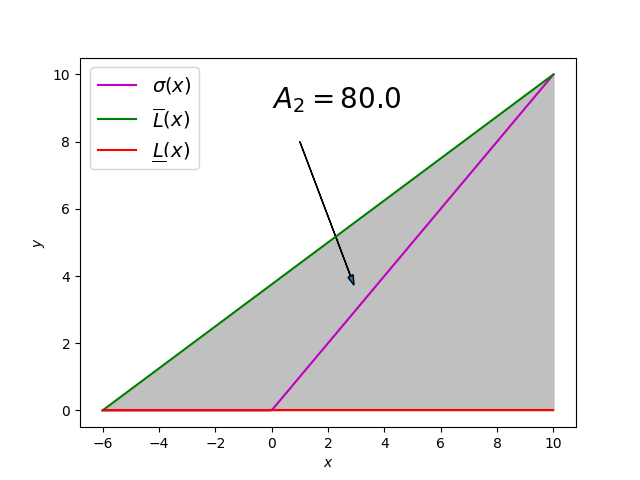} 
	\includegraphics[width=0.32\textwidth, height = 0.17\textheight]{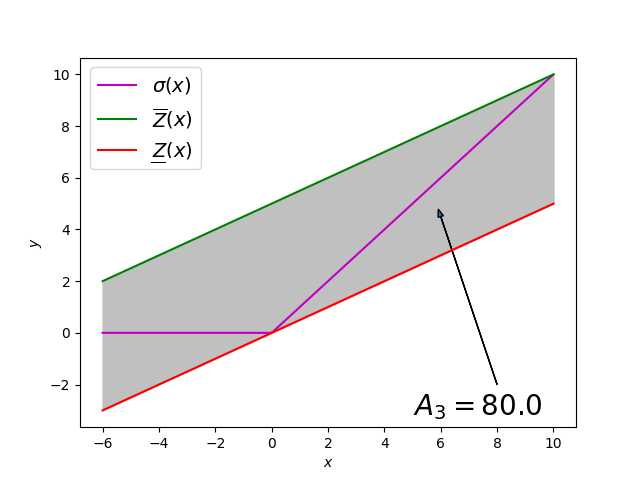} 
	\caption{Illustrations of the over-approximation sets (shaded in gray) of the ReLU activation functions in the interval $\big[-6, 10\big]$ using different approaches: Bernstein approach (Left), triangulation approach (Center), and zonotope approach (Right). Green (Red)-colored curves represent the over-approximation (under-approximation) curves for every approach, respectively. $A_i$, $i \in \{1, 2, 3\}$, represents the over-approximation set's area for every approach.
	}
	\label{fig:ber_trian_zonot}
\end{figure*}

\begin{table}[ht] 
\caption{The area of the over-approximation set of the ReLU activation functions in the interval $\big[-6, 10\big]$ using different Bernstein approach for different approximation order $L$.}
\label{bern_area_vs_order}
\begin{adjustbox}{center}
\begin{tabular}{|c|c|c|c|c|c|}
    \hline
     Approx. & Triangulation & Zonotope & 
     \multicolumn{3}{c|}{Bernstein poly}
     \\ \cline{4-6}
     Method & & & $L = 2$ & $L = 3$ & $L = 8$
     \\ \hline
     error & 80.0 & 80.0 & 37.5 & 28.1 &  16.9\\
     \hline
\end{tabular}
\end{adjustbox}
\end{table}

\subsection{Comparing Bernstein Approximation Against Widely Used Approximations}
The major advantage of using Bernstein polynomials is that they produce a tighter approximation for the response function of ReLU compared to the other state-of-the-art techniques. In particular, existing techniques focus on ``convexifying'' the response of the ReLU function through linear approximation/triangulation (Figure \ref{fig:ber_trian_zonot}-middle) or zonotopes (Figure \ref{fig:ber_trian_zonot}-right). Unlike these techniques, Bernstein polynomials lead to tighter non-convex approximations of the non-convex ReLU function. While it is direct to obtain a closed-form expression for the difference in the approximation error between Bernstein polynomials and triangulation/zonotope approximations, we, instead support our conclusions with the numerical example shown in Table~\ref{bern_area_vs_order} and highlighted in Figure \ref{fig:ber_trian_zonot}. In this example, we compute the approximation error (highlighted in gray) which captures the quality of the over and under-approximations. As captured by this example, it is direct to see that Bernstein polynomials lead to tighter approximation. Moreover, such approximation gets tighter as the approximation order $L$ increases.



\section{Encoding Basic Bernstein Polynomial Operations Using Multi-Dimensional Tensors}

While using Bernstein polynomials to approximate individual ReLU functions provides tighter bounds compared to other techniques, computing Bernstein polynomials via its definition in~\eqref{bernpol}-\eqref{bernpolcoeff2} is time-consuming. That is why state-of-the-art techniques have focused on linear (or convex) relaxations to obtain tractable computations. Nevertheless, in this section, we show that technological advances in Graphics Processing Units (GPUs) can be used to perform all the required operations to efficiently compute Bernstein polynomial approximations of individual neurons along with propagating these polynomials from one layer of the neural network to the next layer. Our main contribution of this section is to encode all necessary operations over Bernstein polynomials into additions and multiplication of multi-dimensional tensors that can be easily performed using GPUs.



\subsection{Multi-dimensional tensor representation of Bernstein polynomials}
We represent the Bernstein polynomial: 
$$B_{g, L}\left(x\right) = \sum\limits_{K \leq L}b^{g}_{K, L}Ber_{K, L}\left(x\right)$$
of function $g$ and order $L$ as a multi-dimensional tensor $\text{Ten}(B_{g, L})$ of $n$ dimensions, and of a shape of $L = \left(l_1 + 1, \cdots, l_n + 1\right)$, where the $K = \left(k_1, \cdots, k_n\right)$ component of $\text{Ten}(B_{g, L})$ is equal to the Bernstein coefficient $b^{g}_{K, L}$. The multi-dimensional tensor $\text{Ten}(B_{g, L})$ represent all the Bernstein coefficients $b^{g}_{K, L}$ of $g$, $\forall K \leq L$. 

\begin{example}
Consider the two-dimensional Bernstein polynomial:
$$B_{g, L}\left(x_1, x_2\right) = \sum\limits_{k_1 = 0}^{2} \sum\limits_{k_2 = 0}^{3} b^{g}_{\left(k_1, k_2\right), L}Ber_{\left(k_1, k_2\right), L}\left(x_1, x_2\right)$$
with orders $L = \left(2, 3\right)$. Its two-dimensional tensor representation is written as follows:
\begin{align}\label{ten_rep}
    \text{Ten}\left(B_{g, L}\right) = 
\begin{bmatrix}
b^{g}_{\left(0, 0\right), L} & b^{g}_{\left(0, 1\right), L} & b^{g}_{\left(0, 2\right), L} & b^{g}_{\left(0, 3\right), L}\\
b^{g}_{\left(1, 0\right), L} & b^{g}_{\left(1, 1\right), L} & b^{g}_{\left(1, 2\right), L} & b^{g}_{\left(1, 3\right), L}\\
b^{g}_{\left(2, 0\right), L} & b^{g}_{\left(2, 1\right), L} & b^{g}_{\left(2, 2\right), L} & b^{g}_{\left(2, 3\right), L}
\end{bmatrix}.
\end{align}
\end{example}

In a similar manner, we represent a multi-dimensional polynomial of order $L$ written in the power series form $p\left(x\right) = \sum\limits_{K \leq L}^{}a_Kx^K$ as a multi-dimensional tensor $\text{Ten}\left(p\right)$ of $n$ dimensions, and of a shape of $L = \left(l_1 + 1, \cdots, l_n + 1\right)$, where the $K = \left(k_1, \cdots, k_n\right)$ component of $\text{Ten}\left(p\right)$ is equal to the coefficient $a_{K}$.

\subsection{Multiplication of two multi-variate Bernstein polynomials}
Multiplying two polynomials represented in the power series form on GPUs has been widely studied in the literature. Unlike power series, multiplying two Bernstein polynomials need extra handling~\cite{scaled_bernstein}. In this subsection, we propose how to encode the multiplication of Bernstein polynomials using GPU implementations that were designed for power-series polynomials.

Given two multivariate polynomials written in a power series form, $p_1 = \sum\limits_{K \leq L_1}^{}a^1_Kx^K$ and $p_2 = \sum\limits_{K \leq L_2}^{}a^2_Kx^K$, and their tensor representation, $\text{Ten}\left(p_1\right)$ and $\text{Ten}\left(p_2\right)$, we use an efficient algorithm \cite{multi_var_poly_prod} that performs multivariate polynomial multiplications. We denote by $\textbf{Prod}\left(\text{Ten}\left(p_1\right), \text{Ten}\left(p_2\right)\right)$ the tensor resulting from such multiplication, i.e.:
%
\begin{align*}
    \text{Ten}\left(p_1p_2\right) = \textbf{Prod}\left(\text{Ten}\left(p_1\right), \text{Ten}\left(p_2\right)\right).
\end{align*}


Applying power-series-based algorithms to multiply two Bernstein polynomials produce incorrect results. Different algorithms were proposed for the case when the Bernstein polynomials are functions of one variable $x_1$~\cite{algs_berns_farouki} and two variables $x_1, x_2$~\cite{scaled_bernstein}. Below, we generalize the procedure in~\cite{scaled_bernstein} to account for Bernstein polynomials in $n$ variables.

\begin{proposition} \label{prop:multiply}
Given two multivariate Bernstein polynomials $B_{g_1, L_1}\left(x\right) = \sum\limits_{K \leq L_1}b^{g_1}_{K, L_1}Ber_{K, L_1}\left(x\right)$ and $B_{g_2, L_2}\left(x\right)= \sum\limits_{K \leq L_2}b^{g_2}_{K, L_2}Ber_{K, L}\left(x\right)$. The tensor representation of the Bernstein polynomial $B_{g_1, L_1}(x) B_{g_2,L_2}(x)$ can be computed as follows:
\begin{align}
    \text{Ten}\left(\tilde{B}_{g_1, L_1}\right) &= \text{Ten}\left(B_{g_1, L_1}\right) * C_{L_1}, \label{scaled_bern_tensor1}\\
    \text{Ten}\left(\tilde{B}_{g_2, L_2}\right) &= \text{Ten}\left(B_{g_2, L_2}\right) * C_{L_2}, \label{scaled_bern_tensor2}\\
  \text{Ten}\left(B_{g_1, L_1} B_{g_2,L_2}\right) &= \frac{1}{C_{L_1 + L_2}} * \textbf{Prod}\left(\text{Ten}\left(\tilde{B}_{g_1, L_1}\right), \text{Ten}\left(\tilde{B}_{g_2, L_2}\right)\right). \label{prod_original_bern}
\end{align}
where $C_{L}$ is the multi-dimensional binomial tensor where its $K$th component is equal to $L \choose K$, i.e, $\left(C_{L}\right)_{K} = {L \choose K}$. With some abuse of notation, we use $1/C_{L}$ to denote the multi-dimensional binomial tensor where its $K$th component is equal to $\frac{1}{{L \choose K}}$. 
\end{proposition}
The proof of Proposition~\ref{prop:multiply} generalizes the argument in~\cite{scaled_bernstein} to multi-dimensional inputs and is omitted for brevity. The Bernstein polynomials in~\eqref{scaled_bern_tensor1} and~\eqref{scaled_bern_tensor2} are called scaled Bernstein polynomials~\cite{scaled_bernstein} and enjoy the fact that their multiplication corresponds to the multiplication of power series polynomials. Hence we can use the power series \textbf{Prod} in~\eqref{prod_original_bern} followed by the element-wise multiplication with the $\frac{1}{C_{L_1 + L_2}}$ tensor to remove the effect of the scaling. Recall that we use $A * B$ to denote the element-wise multiplication between the tensors $A$ and $B$, which can also be carried over using GPUs efficiently which renders all the steps in equations~\eqref{scaled_bern_tensor1}-\eqref{prod_original_bern} to be efficiently implementable on GPUs. We refer to the equations~\eqref{scaled_bern_tensor1}-\eqref{prod_original_bern} as $\textbf{Prod\_Bern}(B_{g_1, L_1}, B_{g_2, L_2})$. 

Using \textbf{Prod\_Bern}, one can compute the tensor corresponding to raising the function $g$ to power $i$, where $i \in \mathbb{N}$ is an integer power, denoted by $\text{Ten}(B_{g^i, L})$ by applying the \textbf{Prod\_Bern} procedure $i$ times. We refer to this procedure as $\textbf{Pow\_Bern}(\text{Ten}(B_{g, L}), i)$.

\subsection{Addition between two Bernstein polynomials}
The authors in~\cite{algs_berns_farouki} studied how to add two Bernstein polynomials. However, their study is restricted to one-dimensional polynomials which are defined over the unity interval $I_1\left(x\right) = [0, 1]$. We extend the argument to the general case with $n$ inputs and any interval $I_n(\underline{d},\overline{d})$ using the following result.

\begin{proposition} \label{prop:addition}
Given two Bernstein polynomials $B_{g_1, L_1}\left(x\right)$ and $B_{g_2, L_2}(x)$ with two different orders $L_1 = \left(l^{1}_1, \cdots, l^{1}_n\right)$ and $L_2 = \left(l^{2}_1, \cdots, l^{2}_n\right)$. Define $L_{sum} = \max(L_1, L_2)$, where the $\max$ operator is applied element-wise.
The tensor representation of $B_{g_1 + g_2, L_{sum}}$ can be computed as:
\begin{align}
    L_{sum} &= (\max(l^{1}_1, l^{2}_1), \ldots, \max(l^{1}_n,l^{2}_n)) \label{eq:L_sum}\\
    \text{Ten}\left(B_{g_1, L_{sum}}\right) &= \textbf{Prod\_Bern}\left(\text{Ten}\left(B_{g_1, L_1}\right), 1_{L_{sum} - L_1 + 1} \right) \label{eq:degree_elevation_1}\\
    \text{Ten}\left(B_{g_2, L_{sum}}\right) &= \textbf{Prod\_Bern} \left(\text{Ten}\left(B_{g_2, L_2}\right), 1_{L_{sum} - L_2 + 1} \right) \label{eq:degree_elevation_2}\\
    \text{Ten}\left(B_{g_1 + g_2, L_{sum}}\right) &= \text{Ten}\left(B_{g_1, L_{sum}}\right) + \text{Ten}\left(B_{g_2, L_{sum}}\right) \label{eq:tensor_sum}
\end{align}
where $1_{L_e - L + 1}$ is a multi-dimensional tensor of a shape $L_e - L + 1$ that contains just ones.
\end{proposition}
The proof of Proposition~\ref{prop:addition} generalizes the argument in~\cite{algs_berns_farouki} and is omitted for brevity. The operation in~\eqref{eq:degree_elevation_1} and~\eqref{eq:degree_elevation_2} is referred to as \emph{degree elevation} in which we change the dimensions of the tensors ... Once both tensors are of the same dimension, we can add them element-wise. We denote by \textbf{Sum\_Bern} the procedure defined by~\eqref{eq:L_sum}-\eqref{eq:tensor_sum}. Again, we note that all the operations in the \textbf{Sum\_Bern} entail tensor element-wise multiplication and addition

\section{BERN-NN algorithm}
In this section, we provide the details of our tool, named BERN-NN. BERN-NN uses the tensor encoding discussed in Section 4 to propagate Bernstein polynomials that over- and under-approximate the different neurons in the network until over- and under-approximation polynomials for the final output of the network are computed.

\subsection{Propagating bounds through single neuron}

We first discuss how to propagate over- and under-approximations through neurons. Recall our notation that we use $\hat{z}^{(i)}_j$ and $z^{(i)}_j$ to denote the input and output of the $j$-th neuron in the $i$-th layer. For ease of notation, we drop the $i$ and $j$ from the notation in this subsection.

Assume that we already computed the over- and under-approximations for the input of one of the hidden neurons, denoted by $\overline{B}_{\hat{z},L_{\hat{z}}}(x)$ and $\underline{B}_{\hat{z},L_{\hat{z}}}(x)$, respectively. The objective is to compute the over- and under-approximations for the output of such a neuron, denoted by $\overline{B}_{{z},L_z}(x)$ and $\underline{B}_{{z},L_z}(x)$, respectively. We proceed as follows.


\noindent \textbf{Step 1: Compute input bounds for the neuron.} Recall that the Bernstein coefficients depend on the input bounds of the function it aims to approximate. Since our aim is to approximate the scalar ReLU function of a neuron, we start by computing the bounds on the input to that neuron as follows:
\begin{align} \label{eq:find_lo_hi}
    lo &= \min_{x \in I_n(\underline{d},\overline{d}) } \underline{B}_{\hat{z},L_{\hat{z}}}(x), \qquad
    hi = \max_{x \in I_n(\underline{d},\overline{d}) } \overline{B}_{\hat{z},L_{\hat{z}}}(x) 
\end{align} 
Thanks to the enclosure property~\eqref{range_enclosing_property}, we can solve the optimization problems~\eqref{eq:find_lo_hi} by finding the minimum and the maximum coefficients of $\underline{B}_{\hat{z},L_{\hat{z}}}$ and $\overline{B}_{\hat{z},L_{\hat{z}}}$.

\noindent \textbf{Step 2: Compute the polynomials $\overline{B}_{\sigma,L}$ and $\underline{B}_{\sigma,L}$ that approximate the ReLU function.}
Given a user-defined approximation order $L$, the next step is to compute the Bernstein polynomials that over- and under-approximate  the ReLU activation function $\sigma$ denoted by $\overline{B}_{\sigma,L}$ and $\underline{B}_{\sigma,L}$. These polynomials can be computed using the knowledge of $lo$ and $hi$ along with the definition of the Bernstein polynomial in~\eqref{bernpolcoeff2}. To facilitate the computations of the next step, we need to convert these polynomials into the corresponding power series form. This can be done by following the procedure in~\cite{shash} to obtain:
\begin{align}
    p_{\overline{B}_{\sigma,L}}(x) &= \sum_{K \le L} a^{\overline{B}_{\sigma,L}}_K x^K, \qquad 
    p_{\underline{B}_{\sigma,L}}(x) = \sum_{K \le L} a^{\underline{B}_{\sigma,L}}_K x^K
\end{align}

\noindent \textbf{Step 3: Propagate the bounds through the decomposition of polynomials.}
First, note that the following holds due to the monotonicity of the ReLU function $\sigma$ and the fact that $z = \sigma(\hat{z})$:
\begin{align}
    \underline{B}_{\hat{z},L_{\hat{z}}}(x) \le \hat{z}(x) &\le \overline{B}_{\hat{z},L_{\hat{z}}}(x) \Rightarrow \\
    &\underbrace{\sigma\left(\underline{B}_{\hat{z},L_{\hat{z}}}(x)\right)}_{\underline{B}_{z,L_z}(x)} \le \underbrace{\sigma\Big(\hat{z}(x)\Big)}_{z(x)} \le \underbrace{\sigma\left(\overline{B}_{\hat{z},L_{\hat{z}}}(x)\right)}_{\overline{B}_{z,L_z}(x)}
\end{align}
In other words, the post-bounds of the neuron, denoted by $\overline{B}_{{z},L_z}(x)$ and $\underline{B}_{{z},L_z}(x)$ can be computed by composing the function $\sigma$ with the under- and over-approximations of the neuron input $\underline{B}_{\hat{z},L_{\hat{z}}}(x)$ and $\overline{B}_{\hat{z},L_{\hat{z}}}(x)$. Indeed such composition is hard to compute due to the nonlinearity in $\sigma$. Instead, we perform such composition with the over- and under-approximations of $\sigma$, $p_{\overline{B}_{\sigma,L}}$ and $p_{\underline{B}_{\sigma,L}}$, computed in Step 2, as:
\begin{align}
    \underline{B}_{{z},L_z}(x) &= 
    \sum_{K \le L} a^{\underline{B}_{\sigma,L}}_K \left(\underline{B}_{\hat{z},L_{\hat{z}}}(x)\right)^K \label{eq:post_bounds_1}\\
    \overline{B}_{{z},L_z}(x) &= 
    \sum_{K \le L} a^{\overline{B}_{\sigma,L}}_K \left(\overline{B}_{\hat{z},L_{\hat{z}}}(x) \right)^K
    \label{eq:post_bounds_2}
\end{align}
Given the tensor representation $Ten(\underline{B}_{\hat{z},L_{\hat{z}}})$ and $Ten(\overline{B}_{\hat{z},L_{\hat{z}}})$, we can use the \textbf{Pow\_Bern} and \textbf{Sum\_Bern} procedures to perform the computations in~\eqref{eq:post_bounds_1} and~\eqref{eq:post_bounds_2} to calculate  $Ten(\underline{B}_{z,L_z})$ and $Ten(\overline{B}_{z,L_z})$ with $L_z = L_{\hat{z}} * L$.

\subsection{Propagating the bounds through one layer}

Next, we discuss how to propagate the under- and over-approximation polynomials of the outputs of the $i-1$ layer denoted by $\underline{B}_{z^{(i-1)}_j,\; L_z}, \overline{B}_{z^{(i-1)}_j, \; L_z}, j \in \{1,\ldots,h_{i-1}\}$ to compute under- and over-approximation of the inputs of the neurons in the $i$th layer $\underline{B}_{\hat{z}^{(i)}_m, \; L_{\hat{z}}}, \overline{B}_{z^{(i)}_m, \; L_{\hat{z}}}, m \in \{1,\ldots,h_{i}\}$ of the neural network. Such bound propagation entails composing the under- and over-approximation polynomials $\underline{B}_{z^{(i-1)}_j, \; L_z}, \overline{B}_{z^{(i-1)}_j, \; L_z}$ with the weights of the $i$th layer of the neural network $W^{(i)}, b^{(i)}$.
To that end, we define the set of positive and negative weights as:
$$W^{(i)}_{+} = \max\left(W^{\left(i\right)}, 0_{i \times (i - 1)}\right) \qquad W^{(i)}_{-} = \min\left(W^{\left(i\right)}, 0_{i \times (i - 1)}\right).$$
Similarly, for the outputs of the $i-1$ layer of the network, we define the vector of over-approximation polynomials and vector of the under-approximation polynomials as:
\begin{align*}
\overline{B}_{z^{(i-1)}, \; L_z} &= \left[ \overline{B}_{z^{(i-1)}_1, L_z} \ldots, \overline{B}_{z^{(i-1)}_{h_{i-1}},L_z}\right]^T, \\
\underline{B}_{z^{(i-1)}, \; L_z} &= \left[ \underline{B}_{z^{(i-1)}_1, L_z} \ldots, \underline{B}_{z^{(i-1)}_{h_{i-1}},L_z}\right]^T,
\end{align*}
and for the inputs of the $i$the layer as:
\begin{align*}
\overline{B}_{\hat{z}^{(i)}, \; L_{\hat{z}}} &= \left[ \overline{B}_{\hat{z}^{(i)}_1, L_{\hat{z}}} \ldots, \overline{B}_{\hat{z}^{(i)}_{h_{i}}, L_{\hat{z}}}\right]^T \\
\underline{B}_{\hat{z}^{(i)}, \; L_{\hat{z}}} &= \left[ \underline{B}_{\hat{z}^{(i)}_1, L_{\hat{z}}} \ldots, \underline{B}_{\hat{z}^{(i)}_{h_{i}}, L_{\hat{z}}}\right]^T
\end{align*}
Hence, the over- and under-approximations of the inputs of the $i$th layer can be efficiently computed as:
\begin{align}
    Ten\left( \overline{B}_{\hat{z}^{(i)}, L_{\hat{z}}} \right) &\!=\! Ten \left( \overline{B}_{z^{(i-1)}, L_z} \right) \!*\! W^{(i)}_{+} \!+\! Ten \left( \underline{B}_{z^{(i-1)}, L_z} \right) \!*\! W^{(i)}_{-} + b^{(i)} \label{eq:upper_bounds_layer}\\
    Ten\left( \underline{B}_{\hat{z}^{(i)}, L_{\hat{z}}} \right) &\!=\! Ten \left( \underline{B}_{z^{(i-1)}, L_z} \right) \!*\! W^{(i)}_{+} \!+\! Ten \left( \overline{B}_{z^{(i-1)}, L_z} \right) \!*\! W^{(i)}_{-} + b^{(i)} \label{eq:lower_bounds_layer}
\end{align}

\subsection{Mechanism of BERN-NN Polynomial Interval Arithmetic}

\begin{figure*}[!t]
\centering
	\includegraphics[scale=0.35]{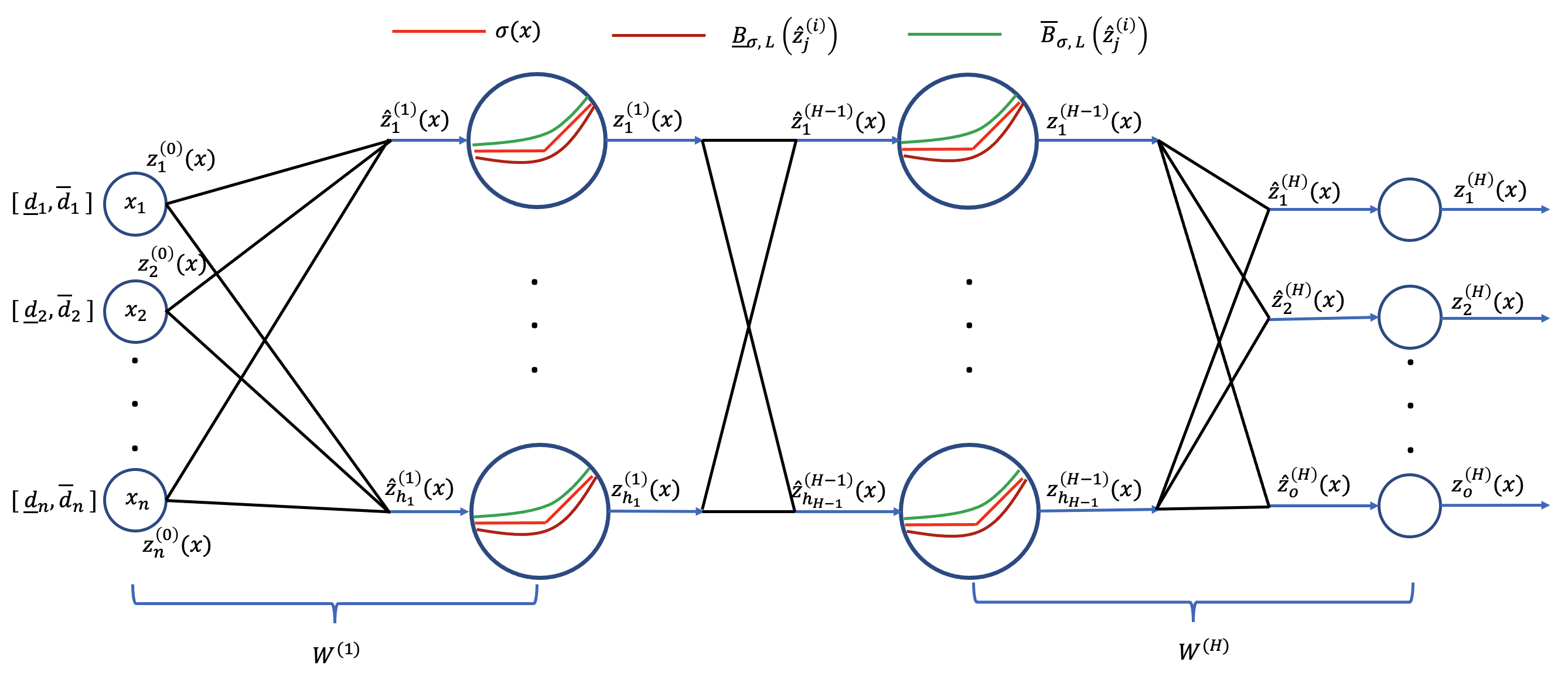} 
	\caption{Mechanism of BERN-NN Polynomial Interval Arithmetic.}
	\label{fig:BERN-NN}
\end{figure*}

We finally describe the proposed BERN-NN Polynomial Interval Arithmetic algorithm, depicted in Figure~\ref{fig:BERN-NN}. 
For a neural network with $n$ inputs $x_1, \ldots, x_n$, we initialize an over- and under-approximation Bernstein polynomials for each of the inputs, i.e.,:
$$ \overline{B}_{z^{(0)}_i, 1} = \underline{B}_{z^{(0)}_i, 1} = \overline{B}_{z^{(0)}_i,1} \qquad i \in \{1, \ldots, n\}.$$
Note that in the equation above, we used $z^{(0)}_i$ as a replacement of $x_i$ to unify the notation with the remainder of the operations (see Figure~\ref{fig:BERN-NN}). To compute the Bernstein polynomials $\overline{B}_{z^{(0)}_i, 1}$ and $\underline{B}_{z^{(0)}_i, 1}$, we recall that the coefficients of such polynomials depend on the input domain. Hence, given a hypercube $I_n(\underline{d}, \overline{d})$ that bounds the input $x$ of the neural network, we compute the tensor representation of these polynomials as:
\begin{align}
    Ten\left( \overline{B}_{z^{(0)}_1, 1} \right) &= Ten\left( \underline{B}_{z^{(0)}_1, 1} \right) = 
    \begin{bmatrix}
    \underline{d}_1 \\
    \overline{d}_1
    \end{bmatrix}
    \otimes 
    \begin{bmatrix}
    1 \\
    1
    \end{bmatrix}
    \otimes
    \ldots
    \otimes
    \begin{bmatrix}
    1 \\
    1
    \end{bmatrix}\\
    Ten\left( \overline{B}_{z^{(0)}_2, 1} \right) &= Ten\left( \underline{B}_{z^{(0)}_2, 1} \right) = 
    \begin{bmatrix}
    1 \\
    1
    \end{bmatrix}
    \otimes
    \begin{bmatrix}
    \underline{d}_2 \\
    \overline{d}_2
    \end{bmatrix}
    \otimes
    \ldots
    \otimes
    \begin{bmatrix}
    1 \\
    1
    \end{bmatrix}\\
    &\vdots\\
    Ten\left( \overline{B}_{z^{(0)}_n, 1} \right) &= Ten\left( \underline{B}_{z^{(0)}_n, 1} \right) = 
    \begin{bmatrix}
    1 \\
    1
    \end{bmatrix}
    \otimes
    \begin{bmatrix}
    1 \\
    1
    \end{bmatrix}
    \otimes
    \ldots
    \otimes
    \begin{bmatrix}
    \underline{d}_n \\
    \overline{d}_n
    \end{bmatrix}
\end{align}





Next, we propagate these over- and under-approximation polynomials to the inputs of the first layer in the neural network using~\eqref{eq:upper_bounds_layer} and~\eqref{eq:lower_bounds_layer}. Given a user-defined approximation order $L$, we propagate the polynomial approximations through the ReLU function using~\eqref{eq:post_bounds_1} and~\eqref{eq:post_bounds_2} for each of the neurons in layer 1. The produced over- and under-approximations of the outputs of all neurons are aggregated together in one tensor which is then propagated to the next layer. This process continues until we compute the over- and under-approximation polynomials of the outputs of the neural network, denoted by $\overline{B}_{z^{(H)}_j, L^{H - 1}}(x), \underline{B}_{z^{(H)}_j, L^{H - 1}}(x)$ for $j = 1, \ldots, o$. These polynomials are used as the solution of Problem 1.

It is important to note that the final Bernstein polynomials $\overline{B}_{z^{(H)}_j, L^{H - 1}}(x), \underline{B}_{z^{(H)}_j, L^{H - 1}}(x)$ have orders of $L^{H - 1}$ where $L$ is the user-defined order of approximation of the ReLU function and $H$ is the number of layers. This polynomial order increases exponentially with the number of hidden layers. Similarly, the shape of their multi-dimensional tensor representations is equal to $L^{H - 1} + 1$ which increases exponentially with the number of hidden layers. To alleviate this problem, we introduce a parameter called $Lin$. Based on this parameter, we drop the orders of the post-bound over- and under-approximation polynomials to $[1, \cdots, 1]$. In other words, we linearize the approximation polynomials every $Lin$ hidden layers. We use the algorithm in \cite{linearize_bernstein} to perform such linearization of the Bernstein polynomial. Luckily, this algorithm, like all the other operations in our BERN-NN involves tensor multiplications and additions and hence can be parallelized over GPUs efficiently.

Finally, note that one can always obtain absolute bounds on the inputs or outputs of any of the neurons (including the outputs of the neural network), thanks to the enclosure property of Bernstein polynomials~\eqref{range_enclosing_property}. Such absolute bounds are useful for reachability analysis and model checkers.

\subsection{GPU Implementation Details}
To get the performance increase of GPUs without the complications of low-level languages, we implemented this tool in PyTorch. As mentioned above, we represent n-dimensional Bernstein polynomials as dense n-dimensional tensors. The tool becomes memory bound very quickly as the number of input nodes increases, making the number of dimensions in the tensors larger. In order to combat this, we use as many in-place operations as possible to avoid repeatedly allocating large chunks of memory during computation. Similarly, the multinomial coefficients used for degree elevation are used multiple times throughout the tool, and we cache each the first time they are generated to avoid spending time re-doing calculations and allocating additional memory.

We parallelized the tool on a node level: at each layer, the outputs of the last layer are passed to each node, which then can run independently of each other on separate GPUs. However, because the tensors become large very quickly, the gains in computation time only offset the overhead of copying tensors between GPUs when the neural network is particularly large. We collect and stack the outputs of all the nodes in one tensor and pass it to the next layer. When the polynomials are being composed with the ReLU approximation, each term is elevated to the highest degree expected of a composition between these two polynomials. This both ensures that the outputs of all the neurons can be stacked, as they are all the same shape and size, and also allows the multiplication of the stacked outputs of the last layer by the incoming weights to be a simple broadcasting multiplication, which is then easily parallelizable on a GPU.

We achieved additional performance gains by rewriting for-loops as element-wise tensor operations and by batching linear algebra operations like matrix multiplications and calculating the least-square solutions of matrices, both of which allow operations to be easily parallelized on GPUs and reduce the amount of time spent allocating many small patches of memory, instead doing a single large allocation.

\section{Numerical Results}
In this section, we perform a series of numerical experiments to evaluate the scalability and effectiveness of  our tool. First, we conduct an ablation study to check the effect of varying different parameters (e.g., neural network width, neural network depth, ReLU approximation order) on the performance of our tool. We utilize two metrics:
\begin{itemize}
    \item \textbf{Execution time}: which measures the time (in seconds) needed to compute the final Bernstein polynomials. Indeed, smaller values indicate better performance.
    \item \textbf{Relative volume of the output set}: this metric measures the ``tightness'' of the produced over- and under-approximation polynomials. Without loss of generality, we focus on neural networks with one output $z^{(H)}$ and we compute this metric as:
    \begin{align}
        \text{Vol\_relative} &= \frac{\text{Vol\_Output}}{\text{Vol\_Input}} \\
         \text{Vol\_Input} &=   \prod_{i= 1}^{n} \bigg( \overline{d}_i - \underline{d}_i\bigg) \\
         \text{Vol\_Output} &=  \idotsint_{I_n} \bigg(\overline{B}_{z^{(H)}}(x) -  \underline{B}_{z^{(H)}}(x)\bigg) \,dx_1 \dots dx_n    
    \end{align}
    Indeed, smaller values of this metric indicate tighter approximations of the output set.
\end{itemize}
 
 After the ablation study, we compare our tool with a set of state-of-the-art bound computation tools---including the winner of the last 2022 Verification of Neural Network (VNN) competition~\cite{bak2021second}---to study the relative performance. 

\textbf{Setup}: We implemented our tool in Python3.9 using PyTorch for all tensor arithmetic. We run all our experiments using a single GeForce RTX 2080 Ti GPU and two 24-core Intel(R) Xeon(R). We like to note that the throughput of the tool can be increased by utilizing multiple GPU to process different neurons in parallel in a batch-processing fashion. However, in this section, we focus on using only one GPU and we leave the generalization of our algorithm to utilize multiple GPUs for future work.  

\subsection{Ablation study}

\subsubsection{The effect of varying the ReLU's order of approximation:}
We study the effect of varying the ReLU's order of approximation $L$ for a fixed NN architecture on the execution time and the output's relative volume space of our tool. In Figure \ref{fig:varying_order_time},  we report the statistical results for 50 random networks of a fixed architecture. Figure \ref{fig:varying_order_time} (top) shows that increasing the approximation order increases the execution time. On the other hand, Figure \ref{fig:varying_order_time} (bottom) shows that the relative volume of the output set significantly decreases with increasing the order of approximation. The results of both figures highlight the trade-off between the tightness of the output bounds and the execution time as a function of the ReLU approximation order $L$.

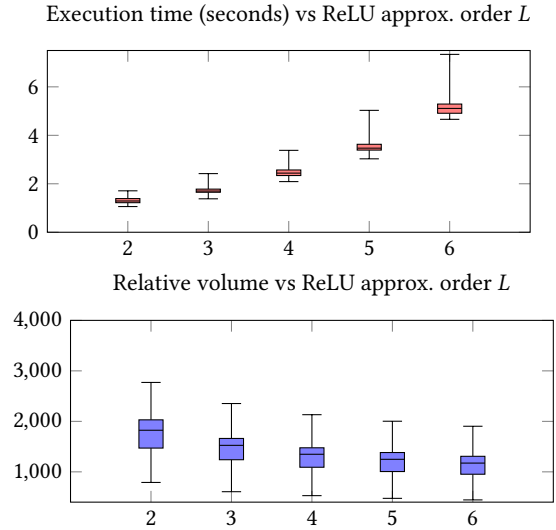
\begin{figure}
    \centering
    \begin{tikzpicture}
        \begin{axis} [
            title= Execution time (seconds) vs ReLU approx. order $L$,
            height=4.0cm, width=8.0cm,
            xmin=0, xmax=6, 
            xtick={1,2,3,4,5},
            xticklabels={2, 3, 4, 5, 6},
            legend style={at={(0.5,-0.2)}, anchor=north, legend columns=-1},
            ymin=0,ymax=7.5,
            boxplot/draw direction=y,
            /pgfplots/boxplot/box extend=0.3,
            boxplot/every box/.style={fill=red!50},
        ]
        \addplot+ [
            color = black,
            boxplot prepared={
              lower quartile=1.22,
               median = 1.29,
              upper quartile=1.39,
              lower whisker=1.06,
              upper whisker=1.71,
            },
        ] coordinates {};
        \addplot+ [
            color = black,
            boxplot prepared={
              lower quartile=1.65,
               median = 1.70,
              upper quartile=1.78,
              lower whisker=1.38,
              upper whisker=2.42,
            },
        ] coordinates {};
        \addplot+ [
            color = black,
            boxplot prepared={
              lower quartile=2.34,
               median = 2.44,
              upper quartile=2.57,
              lower whisker=2.09,
              upper whisker=3.38,
            },
        ] coordinates {};
        \addplot+ [
            color = black,
            boxplot prepared={
              lower quartile=3.39,
               median = 3.47,
              upper quartile=3.63,
              lower whisker=3.03,
              upper whisker=5.03,
            },
        ] coordinates {};
        \addplot+ [
            color = black,
            boxplot prepared={
              lower quartile=4.91,
               median = 5.11,
              upper quartile=5.29,
              lower whisker=4.66,
              upper whisker=7.34,
            },
        ] coordinates {};
        \end{axis}
    \end{tikzpicture}
    \begin{tikzpicture}
        \begin{axis} [
            title=Relative volume vs ReLU approx. order $L$,
            height=4.0cm, width=8.0cm,
            xmin=0, xmax=6, 
            xtick={1,2,3,4,5},
            xticklabels={2, 3, 4, 5, 6},
            legend style={at={(0.5,-0.2)}, anchor=north, legend columns=-1},
            ymin=400,ymax=4000,
            boxplot/draw direction=y,
            /pgfplots/boxplot/box extend=0.3,
            boxplot/every box/.style={fill=blue!50},
        ]
        \addplot+ [
            color = black,
            boxplot prepared={
              lower quartile=1470.7,
               median = 1823.96,
              upper quartile=2030.10,
              lower whisker=790.2,
              upper whisker=2770.94,
            },
        ] coordinates {};
        \addplot+ [
            color = black,
            boxplot prepared={
              lower quartile=1239.97,
               median = 1524.99,
              upper quartile=1661.76,
              lower whisker=606.80,
              upper whisker=2353.04,
            },
        ] coordinates {};
            \addplot+ [
            color = black,
            boxplot prepared={
              lower quartile=1091.03,
               median = 1348.26,
              upper quartile=1476.51,
              lower whisker=528.67,
              upper whisker=2132.04,
            },
        ] coordinates {};
        \addplot+ [
            color = black,
            boxplot prepared={
              lower quartile=1005.86,
               median = 1248.37,
              upper quartile=1382.73,
              lower whisker=475.18,
              upper whisker=2002.61,
            },
        ] coordinates {};
        \addplot+ [
            color = black,
            boxplot prepared={
              lower quartile=952.83,
               median = 1173.46,
              upper quartile=1307.95,
              lower whisker=444.40,
              upper whisker=1903.01,
            },
        ] coordinates {};
        \end{axis}
    \end{tikzpicture}
    \caption{Effect of varying the ReLU's order of approximation $L$ for a NN architecture $[2, 20, 20, 1]$ on the execution time of our tool (top) and the relative volume of the output set (bottom). We set $n = 2$, $I_n = [-1, 1]^n$, and $Lin = 0$. The weights and biases are generated randomly following uniform distribution between $-5$ and $5$. The reported results are generated for $50$ experiments.}
    \label{fig:varying_order_time}
\end{figure}

\subsubsection{The effect of varying the input's dimension:}
We study the effect of varying the input's dimension $n$, for a fixed NN architecture on the execution time of our tool. Figure \ref{fig:varying_dims_time} shows that the execution time for computing the output set grows linearly for smaller values of $n$ but seems to grow more rapidly after $n=7$. This suggests that the proposed tool can be used efficiently for many control applications.

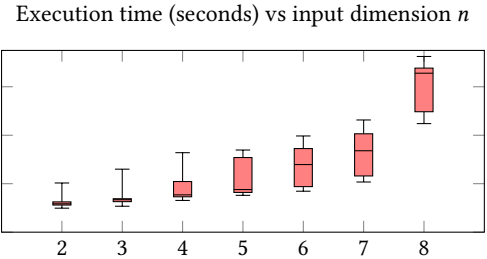
\begin{figure}
    \centering
    \begin{tikzpicture}
        \begin{axis} [
            title=Execution time (seconds) vs input dimension $n$,
            height=4.0cm, width=8.0cm,
            xmin=0, xmax=8, 
            xtick={1,2,3,4,5,6,7},
            xticklabels={2, 3, 4, 5, 6, 7, 8},
            legend style={at={(0.5,-0.2)}, anchor=north, legend columns=-1},
            ymin=0,ymax=7.5,
            boxplot/draw direction=y,
            /pgfplots/boxplot/box extend=0.3,
            boxplot/every box/.style={fill=red!50},
            cycle list={{red},{blue}},
        ]
        \addplot+ [
            color = black,
            boxplot prepared={
              lower quartile=1.12,
               median = 1.17,
              upper quartile=1.25,
              lower whisker=0.99,
              upper whisker=2.03,
            },
        ] coordinates {};
        \addplot+ [
            color = black,
            boxplot prepared={
              lower quartile=1.26,
               median = 1.34,
              upper quartile=1.38,
              lower whisker=1.07,
              upper whisker=2.60,
            },
        ] coordinates {};
        \addplot+ [
            color = black,
            boxplot prepared={
              lower quartile=1.46,
               median = 1.54,
              upper quartile=2.09,
              lower whisker=1.31,
              upper whisker=3.28,
            },
        ] coordinates {};
                \addplot+ [
            color = black,
            boxplot prepared={
              lower quartile=1.64,
               median = 1.759,
              upper quartile=3.08,
              lower whisker=1.52,
              upper whisker=3.39,
            },
        ] coordinates {};
        \addplot+ [
            color = black,
            boxplot prepared={
              lower quartile=1.88,
               median = 2.79,
              upper quartile=3.45,
              lower whisker=1.69,
              upper whisker=3.97,
            },
        ] coordinates {};
        \addplot+ [
            color = black,
            boxplot prepared={
              lower quartile=2.32,
               median = 3.36,
              upper quartile=4.06,
              lower whisker=2.07,
              upper whisker=4.63,
            },
        ] coordinates {};
            \addplot+ [
            color = black,
            boxplot prepared={
              lower quartile=4.97,
               median = 6.56,
              upper quartile=6.77,
              lower whisker=4.48,
              upper whisker=7.25,
            },
        ] coordinates {};
        \end{axis}
    \end{tikzpicture}
    \caption{Effect of varying the input's dimension $n$ for a NN architecture $[n, 20, 20, 1]$ on the execution time our tool. We set $L = 2$, $I_n = [-1, 1]^n$, and $Lin = 0$. The weights and biases are generated randomly following uniform distribution between $-5$ and $5$. The reported results are generated for $50$ experiments.}
    \label{fig:varying_dims_time}
\end{figure}

\subsubsection{The effect of increasing the number of neurons per layer:}

We study the effect of varying the number of neurons per layer $N_e$, for a fixed NN architecture $[3, N_e, N_e, 1]$ on the execution time of our tool. Figure \ref{fig:varying_N_e_time} summarizes the execution times with a varying number of neurons per layer. The results show that increasing the number of neurons per layer highly affects the execution time. This is due to the expensive arithmetic and memory operations for large tensors that represent the Bernstein polynomials. Nevertheless, this increase in execution time can be harnessed by using multiple GPUs to compute bounds for different nodes in parallel along with using the same GPU to process multiple nodes simultaneously.  

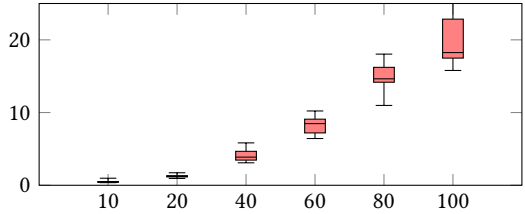
\begin{figure}
    \centering
    \begin{tikzpicture}
        \begin{axis} [
            title=Execution time (seconds) vs number of neurons per layer $N_e$,
            height=4.0cm, width=8.0cm,
            xmin=0, xmax=7, 
            xtick={1,2,3,4,5,6},
            xticklabels={10, 20, 40, 60, 80, 100},
            legend style={at={(0.5,-0.2)}, anchor=north, legend columns=-1},
            ymin=0,ymax=25,
            boxplot/draw direction=y,
            /pgfplots/boxplot/box extend=0.3,
            boxplot/every box/.style={fill=red!50},
            cycle list={{red},{blue}},
        ]
        \addplot+ [
            color = black,
            boxplot prepared={
              lower quartile=0.42,
               median = 0.45,
              upper quartile=0.51,
              lower whisker=0.36,
              upper whisker=0.97,
            },
        ] coordinates {};
            \addplot+ [
            color = black,
            boxplot prepared={
              lower quartile=1.15,
               median = 1.22,
              upper quartile=1.32,
              lower whisker=0.94,
              upper whisker=1.72,
            },
        ] coordinates {};
        \addplot+ [
            color = black,
            boxplot prepared={
              lower quartile=3.46,
               median = 3.87,
              upper quartile=4.67,
              lower whisker=3.09,
              upper whisker=5.83,
            },
        ] coordinates {};
        \addplot+ [
            color = black,
            boxplot prepared={
              lower quartile=7.20,
               median = 8.48,
              upper quartile=9.09,
              lower whisker=6.43,
              upper whisker=10.22,
            },
        ] coordinates {};
        \addplot+ [
            color = black,
            boxplot prepared={
              lower quartile=14.17,
               median = 14.64,
              upper quartile=16.22,
              lower whisker=10.98,
              upper whisker=18.04,
            },
        ] coordinates {};
       \addplot+ [
            color = black,
            boxplot prepared={
              lower quartile=17.49,
               median = 18.25,
              upper quartile=22.85,
              lower whisker=15.79,
              upper whisker=29.14,
            },
        ] coordinates {};
        \end{axis}
    \end{tikzpicture}
    \caption{Effect of varying the number of neurons per layer $N_e$ for a NN architecture $[2, N_e, N_e, 1]$ on the execution time of our tool. We set $n = 2$, $L = 2$, $I_n = [-1, 1]^n$, and $Lin = 0$. The weights and biases are generated randomly following uniform distribution between $-5$ and $5$. The reported results are generated for $50$ experiments.}
    \label{fig:varying_N_e_time}
\end{figure}

\subsubsection{The effect of increasing the number of hidden layers:}

We study the effect of varying the number of hidden layers $n_h$, with $20$ neurons in every hidden layer, on the execution time of our tool. Unlike the effect of increasing the number of neurons per layer, the results in Figure \ref{fig:varying_n_h_time} show that the execution time almost grows linearly with the number of hidden layers.

\begin{figure}
    \centering
    \begin{tikzpicture}
        \begin{axis} [
            title=Execution time (seconds) vs number of layers $n_h$,
            height=4.0cm, width=8.0cm,
            xmin=0, xmax=6, 
            xtick={1,2,3,4,5},
            xticklabels={1, 2, 3, 4, 5},
            legend style={at={(0.5,-0.2)}, anchor=north, legend columns=-1},
            ymin=0,ymax=7,
            boxplot/draw direction=y,
            /pgfplots/boxplot/box extend=0.3,
            boxplot/every box/.style={fill=red!50},
        ]
       \addplot+ [
            color = black,
            boxplot prepared={
              lower quartile=0.27,
               median = 0.28,
              upper quartile=0.31,
              lower whisker=0.25,
              upper whisker=0.63,
            },
        ] coordinates {};
        \addplot+ [
            color = black,
            boxplot prepared={
              lower quartile=1.16,
               median = 1.27,
              upper quartile=1.37,
              lower whisker=0.95,
              upper whisker=2.43,
            },
        ] coordinates {};
        \addplot+ [
            color = black,
            boxplot prepared={
              lower quartile=2.102,
               median = 2.291,
              upper quartile=2.903,
              lower whisker=1.91,
              upper whisker=3.7,
            },
        ] coordinates {};
        \addplot+ [
            color = black,
            boxplot prepared={
              lower quartile=2.97,
               median = 3.23,
              upper quartile=4.11,
              lower whisker=2.60,
              upper whisker=4.88,
            },
        ] coordinates {};
        \addplot+ [
            color = black,
            boxplot prepared={
              lower quartile=3.84,
               median = 4.25,
              upper quartile=4.97,
              lower whisker=3.48,
              upper whisker=6.89,
            },
        ] coordinates {};
        \end{axis}
    \end{tikzpicture}
    \caption{Effect of varying the number of hidden layers $n_h$, for a NN architecture $[2, 20,..,20, 1]$ with $20$ neurons in every hidden layer on the execution time of our tool. We set $n = 2$, $L = 2$, $I_n = [-1, 1]^n$, and $Lin = 0$. The weights and biases are generated randomly following uniform distribution between $-5$ and $5$. The reported results are generated for $50$ experiments.}
    \label{fig:varying_n_h_time}
\end{figure}
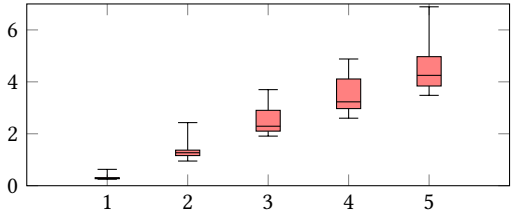

\subsubsection{Scalability analysis of Bern-NN:}
We finally try to study the execution time of Bern-NN for relatively large neural networks. In this study, we add extra layers with 100 neurons each and report the execution time in Figure~\ref{fig:scalability} for random neural networks. As shown in the figure, Bern-NN can process neural networks with more than 1000 neurons in less than 2 minutes.

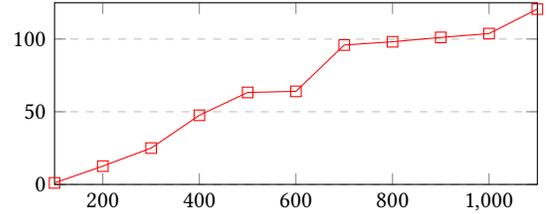
\begin{figure}
    \centering

\begin{tikzpicture}
\begin{axis}[
    title={Execution time (seconds) vs total number of neurons},
    xmin=100, xmax=1100,
    ymin=0, ymax=125,
    ymajorgrids=true,
    grid style=dashed,
    height=4.0cm, width=8.0cm,
]

\addplot[
    color=red,
    mark=square,
    ]
    coordinates {
    (100,1.07)(200,12.6)(300,25.01)(400,47.58)(500,63.27)(600,64.01)(700,95.894)(800,98.13)(900,101.11)(1000,103.75)(1100,120.6)
    };
    
\end{axis}
\end{tikzpicture}
\caption{Scalability of the Bern-NN tool as a function of increasing the total number of neurons.}
    \label{fig:scalability}
\end{figure}

\subsection{Comparison against other tools}
In this subsection, we compare the performance of our tool in terms of execution time and the output set's relative volume compared to bound propagation tools such as Symbolic Interval Analysis (SIA)\cite{wang2018efficient}, alpha-CROWN~\cite{xu2020fast}, and reachability analysis tool such as POLAR~\cite{huang2022polar}. We note that alpha-CROWN~\cite{xu2020fast} was the winner of the 2022 VNN competition and we compare Bern-NN against the bound propagation algorithm used within alpha-CROWN as a representative tool for all the bound propagation techniques. Moreover, alpha-CROWN is also designed to harness the computational powers of GPUs. We compare Bern-NN against POLAR since it also uses polynomials (Taylor Model with a Bernstein error correction) to compute bounds on the output of neural networks. POLAR~\cite{huang2022polar} outperforms other reachability-based tools and hence is a representative tool for such techniques.

\subsubsection{Comparison against SIA and alpha-CROWN for random NN}
We compare the performance of our tool to SIA and alpha-CROWN for random neural networks with $[2, 20, 20, 1]$ architecture for different hyperrectangle input spaces (Figure \ref{fig:compare_tool_inputs_time}). We also compare the performance as the input dimension of the network increases (Figure \ref{fig:compare_tool_dims_time}).
%
The results show that SIA is the fastest in terms of execution time for all different input hyperrectangles due to the simplicity of its computations. However, its relative volume is the highest. On the other hand, Bern-NN's relative volume is the smallest for all different input spaces thanks to its tight higher-order ReLU approximations. Compared to alpha-CROWN (which also runs on GPUs), Bern-NN is both faster and produces tighter bounds leading to an average of $25\%$ reduction in execution time with an average of $10\%$ reduction in the relative volume metric. This shows the practicality of Bern-NN for control applications.

\begin{figure}
    \centering
    \pgfplotsset{width=8cm,height=4cm,compat=1.8}
\begin{tikzpicture}
\begin{axis}[
    ybar,
    enlargelimits=0.15,
    legend style={at={(0.5,1.3)},
      anchor=north,legend columns=-1},
    ylabel={Average execution time (sec)},
    symbolic x coords={input1,input2,input3,input4},
    xtick={input1,input2,input3,input4},
    ]
\addplot coordinates {(input1,0.01) (input2,0.01) (input3,0.01) (input4,0.01)};
\addplot coordinates {(input1,3.52) (input2,3.54) (input3,3.52) (input4,3.53)};
\addplot coordinates {(input1,2.42) (input2,2.50) (input3,2.47) (input4,2.53)};
\legend{SIA,alpha-CROWN,BERN-NN}
\end{axis}
\end{tikzpicture}
\begin{tikzpicture}
\begin{axis}[
    ybar,
    enlargelimits=0.15,
    legend style={at={(0.5,-0.15)},
      anchor=north,legend columns=-1},
    ylabel={Average relative volume},
    symbolic x coords={input1,input2,input3,input4},
    xtick={input1,input2,input3,input4},
    ]
\addplot coordinates {(input1,14974.69) (input2,31073.72) (input3,63195.62) (input4,132461.28)};
\addplot coordinates {(input1,10455.48) (input2,21724.04) (input3,44437.34) (input4,89442.67)};
\addplot coordinates {(input1,9026.93) (input2,18533.17) (input3,37515.83) (input4,77732.05)};
\end{axis}
\end{tikzpicture}    
    \caption{Performance results in terms of average execution times (top) and relative volume (bottom) for BERN-NN, SIA, and alpha-CROWN for different input spaces. The NN's architecture is $[2, 20, 20, 1]$. The ReLU's order of approximation is $L = 4$, and $Lin = 0$. The weights and biases are generated randomly following uniform distribution between $-5$ and $5$. Input1 = $I_n = [-5, 5]^2$, Input2 = $I_n = [-10, 10]^2$, Input3 = $I_n = [-20, 20]^2$, Input4 = $I_n = [-40, 40]^2$.} 

    \label{fig:compare_tool_inputs_time}
\end{figure}
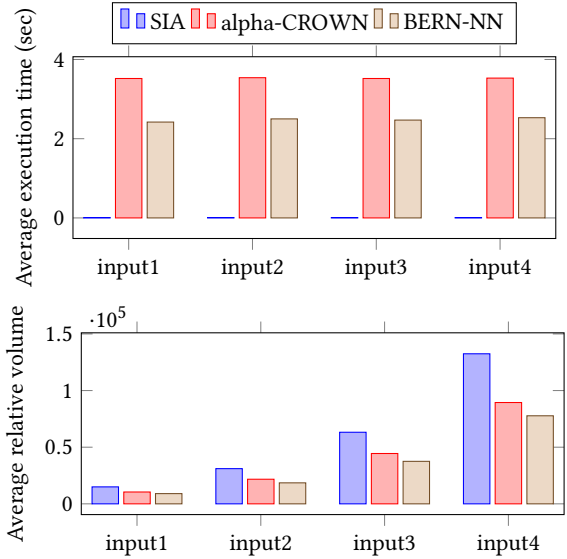



\begin{figure}
    \centering
    \pgfplotsset{width=8cm,height=4cm,compat=1.8}
\begin{tikzpicture}
\begin{axis}[
    ybar,
    enlargelimits=0.15,
    legend style={at={(0.5,1.3)},
      anchor=north,legend columns=-1},
    ylabel={Average execution time (sec)},
    symbolic x coords={dim1,dim2,dim3},
    xtick={dim1,dim2,dim3},
    ]
\addplot coordinates {(dim1,0.01) (dim2,0.01) (dim3,0.01)};
\addplot coordinates {(dim1,3.46) (dim2,3.50) (dim3,3.54)};
\addplot coordinates {(dim1,2.44) (dim2,3.07) (dim3,5.04)};
\legend{SIA,alpha-CROWN,BERN-NN}
\end{axis}
\end{tikzpicture}
\begin{tikzpicture}
\begin{axis}[
    ybar,
    enlargelimits=0.15,
    legend style={at={(0.5,-0.15)},
      anchor=north,legend columns=-1},
    ylabel={Average relative volume},
    symbolic x coords={dim1,dim2,dim3},
    xtick={dim1,dim2,dim3},
    ]
\addplot coordinates {(dim1,31073.72) (dim2,48123.56) (dim3,61135.41)};
\addplot coordinates {(dim1,21724.04) (dim2,34304.63) (dim3,43285.77)};
\addplot coordinates {(dim1,18533.17) (dim2,25924.55) (dim3,31942.63)};
\end{axis}
\end{tikzpicture}    
    \caption{Performance results in terms of average execution times (top) and relative volume (bottom) for BERN-NN, SIA, and alpha-CROWN for input's dimensions $n$. The NN's architecture is $[n, 20, 20, 1]$. the input's space is $[-10, 10]^n$. The ReLU's order of approximation is $L = 4$, $Lin = 0$. The weights and biases are generated randomly following uniform distribution between $-5$ and $5$. dim1 = $n = 2$, dim2 = $n = 3$, dim3 = $n = 4$.}
    \label{fig:compare_tool_dims_time}
\end{figure}
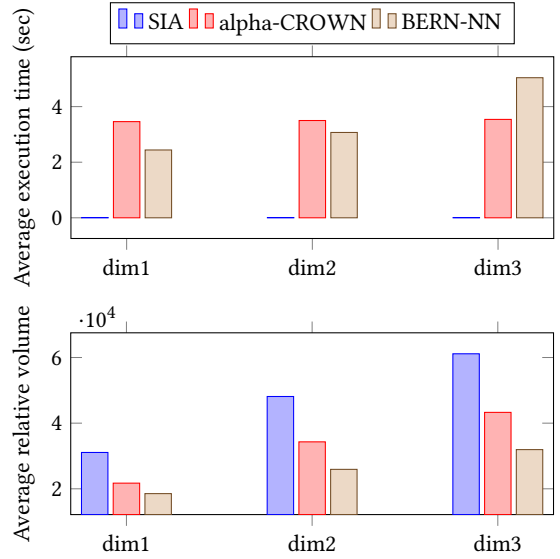


\subsubsection{Case Study for Control Benchmarks} In this experiment, we test different tools on benchmarks of NN controllers (used by POLAR) to evaluate the tightness of their estimated bounds. Table \ref{tab:bench_arch} shows the architecture of the networks used in each benchmark.
Table \ref{tab:tool_comparison_polar} summarizes the performance of the tools with respect to the average execution time and average relative volume for six control benchmarks. The results show that Bern-NN provides the tightest estimate for the output set for all benchmarks except Benchmark 3. We would like to highlight that the tight approximation provided by Bern-NN is important for control applications because the specification of interest is usually defined over a time horizon and require multi-step reachability, hence, tighter bounds at each step are crucial. Lastly, Bern-NN is faster than alpha-CROWN over all benchmarks except Benchmark 5. However, SIA and POLAR are faster than Bern-NN but provide looser bound estimates. Each benchmark is run with five different hyperrectangles that are all centered around zero and have a radius $r \in \{1,1.5,2,2.5,3\}$.

\begin{table*}[ht] 
\caption{Performance results in terms of average execution times and volume for BERN-NN, SIA, alpha-CROWN, and POLAR, for $5$ different input's spaces $I_n\left(\underline{d}, \overline{d}\right)$ for $6$ benchmarks \cite{huang2022polar}. The ReLU's order of approximation is $L = 2$, $Lin = 0$.}

\label{tab:tool_comparison_polar}
\begin{adjustbox}{width=\textwidth,center}
\begin{tabular}{|c|c|c|c|c|c|c|c|c|c|c|c|c|}
    \hline
     \multirow{2}{*}{Tool} & \multicolumn{2}{c|}{Benchmark 1} & \multicolumn{2}{c|}{Benchmark 2} & \multicolumn{2}{c|}{Benchmark 3} & \multicolumn{2}{c|}{Benchmark 4} & \multicolumn{2}{c|}{Benchmark 5} & \multicolumn{2}{c|}{Benchmark 6}  \\
    \cline{2-13}
    &   time & volume & time & volume & time & volume & time & volume & time & volume & time & volume \\
    \hline
    \hline
     $SIA$ & $\bm{0.01}$ & $2.544$ & $\bm{0.02}$ & $6.05$ & $\bm{0.01}$ & $1.02$ & $\bm{0.01}$ & $9.41$ & $\bm{0.02}$ & $53.38$ & $\bm{0.02}$ & $2.03$ \\
   \hline
     $CROWN$  & $2.9$ & $3.1$ & $3.49$ & $5.50$ & $3.54$ & $\bm{0.73}$ & $3.13$ & $17.04$ & $3.80$ & $77.72$ & $4.10$ & $2.4$\\
    \hline 
     $Bern-NN$ & $0.84$ & $ \bm{1.62}$ & $1.30$ & $\bm{5.4}$ & $1.09$ & $0.81$ & $1.15$ & $\bm{6.21}$ & $41.7$ & $\bm{35.85}$ & $3.25$ & $\bm{1.38}$\\
     \hline 
     $POLAR$ & $0.21$ & $25.43$ & $0.284$ & $51.80$ & $0.29$ & $18.81$ & $0.42$ & $33.32$ & $5.52$ & $432.75$ & $0.81$ & $7.00$\\
     \hline 
\end{tabular}
\end{adjustbox}
\end{table*}


\begin{table}[]
\caption{Architectures of POLAR Benchmarks}
\label{tab:bench_arch}
\resizebox{0.6\columnwidth}{!}{%
\begin{tabular}{|l|l|}
\hline
            & Architecture       \\ \hline
Benchmark 1 & {[}2,20,20,1{]}    \\ \hline
Benchmark 2 & {[}2,20,20,1{]}    \\ \hline
Benchmark 3 & {[}2,20,20,1{]}    \\ \hline
Benchmark 4 & {[}3,20,20,1{]}       \\ \hline
Benchmark 5 & {[}3,100,100,1{]}    \\ \hline
Benchmark 6 & {[}4,20,20,20,1{]} \\ \hline
\end{tabular}%
}
\end{table}

\section{Conclusion}

In conclusion, we presented Bern-NN, a tool for computing higher-order tight bounds for NNs by approximating non-linear ReLU activations using Bernstein polynomials. We provided GPU-based computational machinery to handle tensor arithmetic for manipulating polynomials as well as bounding them using the properties of Bernstein polynomials. 
We conducted extensive experiments to evaluate the scalability of our tool as well as compare its estimated bounds with state-of-the-art methods. The results showed that our tool can process neural networks with thousands of neurons in a few minutes. These results also show that our tool outperforms state-of-the-art tools in terms of computing tighter bounds while reducing the execution time compared to other tools.

\bibliographystyle{ACM-Reference-Format}
\bibliography{Bibliography, Hkhedr}

\end{document}